\documentclass[12pt, anon]{l4dc2024}

\title[Submodular Information Selection for Hypothesis Testing with Misclassification Penalties]{Submodular Information Selection for Hypothesis Testing with Misclassification Penalties}
\usepackage{times}
\usepackage{appendix}
\newtheorem*{theorem*}{Theorem}
\newtheorem*{lemma*}{Lemma}
\newtheorem*{corollary*}{Corollary}
\usepackage{enumitem}
\usepackage{algorithm}
\usepackage{algpseudocode}
\usepackage{wrapfig}
\newtheorem{Problem}{Problem}

\newcommand{\ef}{f_{\theta_p}}

\coltauthor{\Name{Jayanth Bhargav} \Email{jbhargav@purdue.edu}\\
 \Name{Mahsa Ghasemi} \Email{mahsa@purdue.edu}\\
 \Name{Shreyas Sundaram} \Email{sundara2@purdue.edu}\\
 \addr Elmore Family School of Electrical \& Computer Engineering, Purdue University, IN 47907 USA}

\begin{document}

\maketitle
\begin{abstract}%
 We consider the problem of selecting an optimal subset of information sources for a hypothesis testing/classification task where the goal is to identify the true state of the world from a finite set of hypotheses, based on finite observation samples from the sources. In order to characterize the learning performance, we propose a misclassification penalty framework, which enables non-uniform treatment of different misclassification errors. In a centralized Bayesian learning setting, we study two variants of the subset selection problem: (i) selecting a minimum cost information set to ensure that the maximum penalty of misclassifying the true hypothesis is below a desired bound and (ii) selecting an optimal information set under a limited budget to minimize the maximum penalty of misclassifying the true hypothesis. Under certain assumptions, we prove that the objective (or constraints) of these combinatorial optimization problems are weak (or approximate) submodular, and establish high-probability performance guarantees for greedy algorithms. Further, we propose an alternate metric for information set selection which is based on the total penalty of misclassification. We prove that this metric is submodular and establish near-optimal guarantees for the greedy algorithms for both the information set selection problems. Finally, we present numerical simulations to validate our theoretical results over several randomly generated instances.
 \end{abstract}
 %
\section{Introduction}
In many autonomous systems, agents depend on predictions made by classifiers for making decisions (or taking actions), and may have to pay a high cost for acting on erroneous predictions. An example of this is an incident of an autonomous vehicle crash caused due to the vision system misclassifying a white truck as a bright sky (\cite{nhsta}). In such scenarios, one needs to ensure minimal risk associated with misclassification. In order to improve the quality of predictions, one may need to select an optimal set of features (or observations), often provided by information sources (or sensors), that can best describe the true state. In many practical scenarios, due to limitations on communication or compute resources, one can only query data from a small subset of information sources (\cite{krause2010submodular,chepuri2014sparsity, hashemi2020randomized}). Moreover, one may also need to pay a certain cost in order to obtain measurements from information sources (\cite{krause2008near}). Thus, a fundamental problem that arises in such scenarios is to select a subset of information sources with minimal cost or under a limited budget, while ensuring certain learning performance using the observations provided by the selected sources. In order to characterize the quality of an information set, we propose a framework based on misclassification penalties, specified by a penalty matrix. The goal is to select an information set that minimizes the maximum penalty of misclassifying the true state. As a motivating example, consider a  surveillance task, where identifying a target of interest is of importance. One many have to pay a penalty for misclassifying the true state, for instance, misclassifying a drone (an intruder) as a bird. However, the event of misclassifying a bird as a drone may have a different penalty associated with it. The penalty matrix captures the fact that different misclassification errors incur different penalties.

\subsection{Related Work}
Misclassification risk and uncertainty quantification for various types of classifiers have been very well studied in the literature (\cite{adams1999comparing,pendharkar2017bayesian,hou2013modeling}). In \cite{sensoy2021misclassification}, the authors propose a risk-calibrated classifier to reduce the costs associated with misclassification errors, and empirically show the effectiveness of their algorithm, in a  deep learning framework. In \cite{elkan2001foundations}, the authors study cost-sensitive learning for class balancing in order to improve the quality of predictions in decision tree learning methods. In our work, we consider a hypothesis testing (or classification) task in a Bayesian learning framework. 

A subset of the literature has addressed the problem of sequential information gathering within a limited budget (\cite{hollinger2013sampling,chen2015sequential}). The authors of  \cite{golovin2010near} study data source selection for a monitoring application, where the sources are selected sequentially in order to estimate certain parameters of an environment. In \cite{ghasemi2019online}, the authors study sequential information gathering under a limited budget for a robotic navigation task.  In contrast, we consider the scenario where the information set is selected \textit{a priori}. 

A substantial body of work focuses on the study of submodularity (and/or weak submodularity) and greedy techniques with provable guarantees for feature selection in sparse learning (\cite{krause2010submodular,chepuri2014sparsity}); sensor selection for estimation (\cite{mo2011sensor,hashemi2020randomized}),  Kalman filtering  (\cite{ye2018complexity}), and mixed-observable Markov decision processes (\cite{bhargav2023complexity}). Along the lines of these works, we leverage the weak submodularity property of the performance metric and present greedy algorithms with performance guarantees.

The closest paper to our work is \cite{ye2021near}, in which the authors studied data source selection for Bayesian learning, where the learning performance was characterized by a total variation error metric based on the asymptotic belief. However, we consider a non-asymptotic setting, where the learning performance is characterized by misclassification penalties. Building upon the results in \cite{ye2021near} and \cite{das2018approximate}, we establish theoretical guarantees for greedy information selection algorithms presented in this paper. 

\subsection{Contributions}  We consider two variants of an information subset selection problem for a hypothesis testing task (i) selecting a minimum cost information set to ensure the maximum penalty for misclassifying the true hypothesis is below a desired bound and (ii) optimal information set selection under a limited budget to minimize the maximum penalty of misclassifying the true hypothesis. First, we prove that the maximum penalty metric is weak submodular by characterizing its submodularity ratio, and establish high-probability guarantees for greedy algorithms for both the problems, along with the associated finite sample convergence rates for the Bayesian beliefs. Next, we propose an alternate metric based on the total penalty of misclassification. We prove that this metric is submodular, and establish  near-optimal guarantees for the greedy algorithms. Finally, we evaluate the empirical performance of the proposed greedy algorithms over several randomly generated problem instances.

\section{Minimum-Cost Information Set Selection Problem}
\label{sec:mcis}
In this section, we formulate the minimum-cost information set selection problem. Let $\Theta = \{ \theta_1, \theta_2, \hdots, \theta_m \}$, where $m = |\Theta| $, be a finite set of possible hypotheses (also referred to as classes or states), of which one of them is the true state of the world. We consider a set $\mathcal{D} = \{1,2, \hdots, n \}$ of information sources (or data streams) from which we need to select a subset $\mathcal{I} \subseteq \mathcal{D}$. At each time step $t \in \mathbb{Z}_{\geq 1}$, the observation provided by the information source $i \in \mathcal{D}$ is denoted as $o_{i,t} \in O_i$, where $O_i$ is the observation space of the source $i$. Each information source $i \in \mathcal{D}$ is associated with an observation likelihood function $\ell_i(\cdot | \theta)$, which is conditioned on the state of the world $\theta \in \Theta$. At any time $t$, conditioned on the true state of the world $\theta \in \Theta$, a joint observation profile of $n$ information sources, denoted as $o_t = (o_{1,t}, \hdots, o_{n,t}) \in \mathcal{O}$ where $\mathcal{O} = O_1 \times \hdots \times O_n$, is generated by the joint likelihood function $\ell(\cdot|\theta)$. We make the following assumption on the observation model (e.g., see \cite{jadbabaie2012non}; \cite{ liu2014social}; \cite{lalitha2014social} for detailed discussions).

\textbf{Assumption 1:} \textit{The observation space $O_i$ associated with each information source $i \in \mathcal{D}$ is finite, and the likelihood function $\ell_i(\cdot | \theta)$ satisfies $\ell_i(\cdot | \theta) > 0$ for all $o_i \in O_i$ and for all $\theta \in \Theta$. We assume that the designer knows $\ell_i(\cdot | \theta)$ for all $\theta \in \Theta$ and all $i \in \mathcal{D}$. For all $\theta \in \Theta$, conditioned on the true state, the observations are independent of each other over time, i.e., $\{ o_{i,1}, o_{i,2}, \hdots \}$ is a sequence of independent identically distributed (i.i.d.) random variables, given a true state $\theta \in \Theta$. }

Consider the scenario where a designer at a central node needs to select a subset of information sources in order to identify the true state of the world. Each source $i \in \mathcal{D}$ has a selection cost $c_i \in \mathbb{R}_{> 0}$. For any subset $\mathcal{I} \subseteq \mathcal{D}$ with $|\mathcal{I}| = k$, let $\{s_1, s_2, \hdots, s_k\}$ denote the set of information sources. The cost of the information set $\mathcal{I}$ is given by $c(\mathcal{I}) = \sum_{s_i \in \mathcal{I}} c_{s_i}$. The joint observation conditioned on the $\theta \in \Theta$ of this information set at time $t$ is defined as $o_{\mathcal{I},t} = \{ o_{s_1,t}, \hdots, o_{s_k,t} \} \in O_{s_1} \times \hdots \times O_{s_k} $, and is generated by the joint likelihood function $\ell_{\mathcal{I}}(\cdot | \theta) = \Pi_{i = 1}^{k} \ell_{s_i}(\cdot | \theta)$ (by Assumption 1), and the central designer knows $\ell_{\mathcal{I}} (\cdot | \theta)$ for all $\mathcal{I} \subseteq \mathcal{D}$ and for all $\theta \in \Theta$.


Assumption 1 also implies the existence of a constant $L \footnote{The constant $L$ is an upper bound on the maximum difference between the  log-likelihood of an observation from an information source under any two hypotheses, which we will use later in our analyses.} \in(0, \infty)$ such that:
\begin{equation}
\label{eq:kld_ratio}
\max _{i \in \mathcal{D}} \max _{o_i \in O_i} \max _{\theta_p, \theta_q \in \Theta}\left|\log \frac{\ell_i\left(o_i \mid \theta_p\right)}{\ell_i\left(o_i \mid \theta_q\right)}\right| \leq L.
\end{equation}

For a true state $\theta_p \in \Theta$, we define $\mathbb{P}^{\theta_p} = \prod_{t=1}^{\infty}\ell(\cdot|\theta_p)$ to be the probability measure. For the sake of brevity, we will say that an event occurs almost surely to mean that it occurs almost surely w.r.t. the probability measure $\mathbb{P}^{\theta_p}.$ 
As the data comes in, the central node updates its belief over the set of possible hypotheses using the standard Bayes' rule. Let $\mu_{t}^{\mathcal{I}}(\theta)$ denote the belief of the central designer (or node) that $\theta$ is the true hypothesis  at time step $t$ based on the information sources in $\mathcal{I}$, and let $\mu_0(\theta)$ denote the initial belief (or prior) of the central node that $\theta$ is the true state of the world, with $\sum_{\theta \in \Theta} \mu_0(\theta) = 1$. The Bayesian update rule is given by 
\begin{equation}
    \label{eq:bayes_full}
    \mu_{t+1}^{\mathcal{I}}(\theta) = \frac{\mu_0(\theta)  \prod_{j=0}^{t} \ell_{\mathcal{I}} (o_{\mathcal{I}, j+1} | \theta) }{\sum_{\theta_i \in \Theta} \mu_0(\theta_i) 
  \prod_{j=0}^{t} \ell_{\mathcal{I}} (o_{\mathcal{I}, j+1} | \theta_i) } \quad \forall \theta \in \Theta.
\end{equation}

For a hypothesis $\theta \in \Theta$ and an information set $\mathcal{I} \subseteq \mathcal{D}$, we have the following.
\begin{definition}[Observationally Equivalent Set](\cite{ye2021near}) For a given hypothesis (or class) $\theta \in \Theta$ and a given $\mathcal{I} \subseteq \mathcal{D}$, the \textit{observationally equivalent set} of classes to $\theta$ is defined as
\begin{equation}
    F_{\theta}(\mathcal{I}) = \{ \theta_i \in \Theta \mid D_{KL} (\ell_\mathcal{I}(\cdot | \theta_i) || \ell_\mathcal{I}(\cdot | \theta) ) = 0 \},
\end{equation}
where $D_{KL}(\ell_\mathcal{I}(\cdot | \theta_i) || \ell_\mathcal{I}(\cdot | \theta) )$ is the Kullback-Leibler divergence measure $\ell_\mathcal{I}(\cdot | \theta_i)$ and $ \ell_\mathcal{I}(\cdot | \theta) $.
\end{definition}

From the definition above, we have $\theta \in F_{\theta}(\mathcal{I})$ for all $\theta \in \Theta$ and all for $\mathcal{I} \subseteq \mathcal{D}$. We can write the set $F_{\theta}(\mathcal{I})$ equivalently as

\begin{equation}
\label{eq:obs_eq_set}
    F_{\theta}(\mathcal{I}) = \{ \theta_i \in \Theta : \ell_{\mathcal{I}}(o_{\mathcal{I}} | \theta_i) = \ell_{\mathcal{I}}(o_{\mathcal{I}} | \theta), \forall o_{\mathcal{I}} \in \mathcal{O}_{\mathcal{I}}   \},
\end{equation}
where $\mathcal{O}_{\mathcal{I}}  = O_{s_1} \times \hdots \times O_{s_k}$ is the joint observation space of the information set $\mathcal{I}$. In other words, $F_{\theta}(\mathcal{I})$ is the set of hypotheses (or classes) that cannot be distinguished from $\theta$ based on the observations obtained by the information sources in $\mathcal{I}$. Furthermore, by Assumption 1 and Equation (\ref{eq:obs_eq_set}), we have the following (see Section 2 in \cite{ye2021near}):

\begin{equation}
    \label{eq: obs_eq_intersect}
    F_{\theta}(\mathcal{I}) = \bigcap_{s_i \in \mathcal{I}} F_{\theta} (s_i), \forall \mathcal{I} \in \mathcal{D}, \forall \theta \in \Theta.
\end{equation}
Define $F_{\theta}(\emptyset) = \Theta$, i.e.,  when there is no information set, all classes are observationally equivalent.

 At time $t$, the central designer predicts the state of the world based on the belief  $\mu_{t}^{\mathcal{I}}$ generated by the information set $\mathcal{I}$. In order to characterize the learning performance, we consider a penalty-based classification framework. Let  $\Xi = [\xi_{ij}] \in \mathbb{R}^{m \times m}$ denote the \textit{penalty matrix}, where $ 0 \leq \xi_{ij} \leq 1$ is the penalty associated with predicting the class to be $\theta_j$, given that the true class  is $\theta_i$. The penalty matrix is assumed to be \textit{row stochastic}, i.e., $\sum_{j = 1}^{m} \xi_{ij} = 1$. We have $\xi_{ii} = 0, \hspace{4pt} \forall i \in \{1,2, \hdots, m\}$, i.e., there is no penalty when the predicted hypothesis is the true hypothesis.  Similar to analyses presented in \cite{nedic2017fast} and \cite{mitra2020new}, we present finite sample convergence rates for the Bayesian belief over the set of hypotheses. In this paper, we consider the case of a uniform prior, but the results can be extended to non-uniform priors (using similar arguments as in Lemma 1 of \cite{mitra2020new}). We defer all the proofs to the Appendix (see Appendix \ref{app:proof}).

\begin{theorem}
\label{thm2}
    Let the true state of the world be $\theta_p$ and let $\mu_0 (\theta) = \frac{1}{m} \hspace{5pt} \forall \theta \in \Theta$ (i.e., uniform prior). Under Assumption 1, for any $\delta, \epsilon \in [0,1]$, and L as defined in Equation (\ref{eq:kld_ratio}), and for an information set $\mathcal{I} \subseteq \mathcal{D}$, the Bayesian update rule in Equation (\ref{eq:bayes_full}) has the following property: there is an integer $N(\delta, \epsilon, L)$, such that with probability at least $1-\delta$, for all $t > N(\delta, \epsilon, L)$ we have: 
    \begin{enumerate}[label=(\alph*)]
        \item  $\mu_t^{\mathcal{I}}(\theta_q) =  \mu_t^{\mathcal{I}}(\theta_p)  \hspace{5pt} \forall \theta_q \in F_{\theta_p}(\mathcal{I})$, and
        \item  $\mu_t^{\mathcal{I}}(\theta_q) \leq \exp{(-t(|K(\theta_p, \theta_q)-\epsilon|))} \hspace{5pt} \forall \theta_q \notin F_{\theta_p}(\mathcal{I})$;
    \end{enumerate}
    where $K(\theta_p, \theta_q) = D_{KL}(\ell_{\mathcal{I}}(\cdot | \theta_p) || \ell_{\mathcal{I}}(\cdot | \theta_q) )$ is the Kullback-Leibler divergence measure between the likelihood functions $\ell_{\mathcal{I}}(\cdot | \theta_p)$ and $ \ell_{\mathcal{I}}(\cdot | \theta_q) $, $F_{\theta_p}(\mathcal{I})$ is defined in (\ref{eq:obs_eq_set}), and 
        $N(\delta, \epsilon,L) = \left\lceil \frac{2L^2}{\epsilon^2 } \log \frac{2}{\delta} \right\rceil.$
\end{theorem}

We consider a belief threshold rule in order to rule out hypotheses that do not have a high likelihood of being predicted as the true hypothesis. Let $\mu_{th}$ be the threshold chosen by the central designer. Corollary \ref{coro1} presents the sample complexity for the observations in order to ensure that the beliefs over the states $\theta_q \notin F_{\theta_p}(\mathcal{I})$ remain bounded under the specified threshold.
\begin{corollary}
\label{coro1}
    Instate the hypothesis and notation of Theorem \ref{thm2}. For a specified threshold $\mu_{th}\in    (0,1)$ for the belief over any class $\theta_q \notin F_{\theta_p}(\mathcal{I}) $, there exists $\delta, \epsilon \in [0,1]$, for which one can guarantee with probability at least $1-\delta$ that $\mu_t^{\mathcal{I}}(\theta_q) \leq \mu_{th}$  for all $\theta_q \notin F_{\theta_p}$ and for all $t > \Tilde{N}$, where 
    \begin{equation}
    \label{eq:nbar}
         \Tilde{N} = \left\lceil  \max \left\{ \frac{2L^2}{\epsilon^2 } \log \frac{2}{\delta} , \frac{1}{ \displaystyle\min_{\theta_p,\theta_q \in \Theta} |K(\theta_p, \theta_q)-\epsilon| } \log \frac{1}{\mu_{th}}   \right\} \right\rceil.
    \end{equation}
\end{corollary}

From Corollary \ref{coro1}, we have the following: After any $t> \Tilde{N}$, the central node will predict one of  $\theta_q \in F_{\theta_p}(\mathcal{I})$ to be the true hypothesis,  with probability at least $1-\delta$. Therefore, it is sufficient to consider the penalties associated with the states $\theta_q \in F_{\theta_p}(\mathcal{I})$ for finding the maximum penalty. We now formalize the Minimum-Cost Information Set Selection (MCIS) Problem as follows:

\begin{Problem}[MCIS]
\label{prob:dsrc}
    Consider a set $\Theta=\left\{\theta_1, \ldots, \theta_m\right\}$ of possible states of the world, a set $\mathcal{D}$ of information sources, a selection cost $c_i \in \mathbb{R}_{>0}$ of each source $i \in \mathcal{D}$, a row-stochastic penalty matrix $\Xi = [\xi_{ij}] \in \mathbb{R}^{m \times m}$, and prescribed penalty bounds $0 \leq R_{\theta_p} \leq 1$ for all $\theta_p \in \Theta$. The MCIS Problem is to find a set of selected information sources $\mathcal{I} \subseteq \mathcal{D}$ that solves
\begin{equation}
\label{eq:dsrc}
\begin{aligned}
 \min _{\mathcal{I} \subseteq \mathcal{D}} & \hspace{5pt} c(\mathcal{I}) ; \hspace{2pt}
 \text { s.t. } \displaystyle \max_{\theta_i \in F_{\theta_p}(\mathcal{I})} \xi_{pi}  \leq R_{\theta_p} \quad \forall \theta_p \in \Theta.
\end{aligned}
\end{equation}
\end{Problem}
    


\subsection{Weak Submodularity and Greedy Algorithm }
\label{sec:weak_submod}
The combinatorial optimization in (\ref{eq:dsrc}) can be shown to be NP-hard (based on similar arguments as in Theorem 3  of \cite{ye2021near}). In this section, we propose a greedy algorithm with performance guarantees to efficiently approximate the solution to the MCIS Problem. We first begin by transforming the MCIS problem into the minimum cost set cover problem studied in \cite{wolsey1982analysis}. 

\begin{definition}[Monotonicity] A set function $f:2^{\Omega }\to\mathbb{R}$ is monotone non-decreasing if $f(X) \leq f(Y) $ for all $X \subseteq Y \subseteq \Omega$ and monotone non-increasing if $f(X) \geq f(Y) $ for all $X \subseteq Y \subseteq \Omega$.
    
\end{definition}

\begin{definition}[Submodularity Ratio]\footnote{There are several notions of submodularity ratio. We consider $\gamma_{U,k}$ as defined in \cite{das2018approximate}, where $U$ is the universal set and $k\ge 1$ is a parameter, and drop the dependence on $k$ by defining $\gamma = \min_{k} \gamma_{U,k}$. }
\label{def:submod_ratio}
 Given a set  $\Omega$, the submodularity ratio of a non-negative function $f: 2^{\Omega} \to \mathbb{R}_{\ge 0}$ is the largest $\gamma \in \mathbb{R}$ that satisfies  for all $A, B \subseteq \Omega$, the following:
 \begin{equation*}
 \label{eq:submod_ratio_def}
     \sum_{a \in A \setminus B} (f(\{a\} \cup B) - f(B)) \ge \gamma (f(A \cup B) - f(B)).
 \end{equation*}
\end{definition}
\begin{remark}
\label{remark:submod_ratio}
    For a non-negative and non-decreasing function $f(\cdot)$ with submodularity ratio $\gamma$, we have $\gamma \in [0,1]$. If $\gamma$ is closer to 1, the function is closer to being submodular. $f(\cdot)$ is submodular if and only if $\gamma \ge 1$. \cite{das2018approximate} provide guarantees for greedy optimization of weak submodular functions, which depend on the submodularity ratio $\gamma$. Thus, in order to characterize the performance of greedy, one has to give a (non-zero) lower bound on $\gamma$.
\end{remark}

The constraint in (\ref{eq:dsrc}) can be equivalently written as: $1- \max_{\theta_i \in F_{\theta_p}(\mathcal{I})} \xi_{pi} \geq 1-R_{\theta_p}, \hspace{4pt} \forall \theta_p \in \Theta.$ For all $\mathcal{I} \subseteq \mathcal{D}$ and for a true state $\theta_p \in \Theta$, let us define $f_{\theta_p}(\mathcal{I}) =1- \max_{\theta_i\in F_{\theta_p}(\mathcal{I})} \xi_{pi}$. It follows from  (\ref{eq: obs_eq_intersect}) that $f_{\theta_p}(\cdot)$ is a monotone non-decreasing set function with $f_{\theta_p}(\emptyset) = 1-\max_{\theta_j\in \Theta} \xi_{pj}$.

In order to establish the approximate (or weak) submodularity property, we make the following assumption on the misclassification penalties. 

\noindent \textbf{Assumption 2:} The misclassification penalties are unique, i.e., $\xi_{pi} \neq \xi_{pj}$ for all $i \neq j, \forall \theta_p \in \Theta$.  \\
\noindent Note that the above assumption requires that no two misclassification events have the same penalty associated with them, which is often a reasonable assumption in many applications.
\begin{lemma} \label{lma:submod} Under Assumption 2, the function $f_{\theta_p}(\mathcal{I}) :2^{\mathcal{D} }\to\mathbb{R}_{\geq0} $ is approximately submodular for all $\theta_p \in \Theta$, with a submodularity ratio $\gamma = \underline{\xi}/ \Bar{\xi} $, where 
\begin{align}
\label{eq:xifloor}
\displaystyle \underline{\xi} = \min_{\theta_p \in \Theta} \left(\min_{\theta_i,\theta_j \in \Theta}  |\xi_{pi} - \xi_{pj}|\right) ; \quad \hspace{2pt}
    \Bar{\xi} = \max_{\theta_p \in \Theta} \left(\max_{\theta_i,\theta_j \in \Theta}  |\xi_{pi} - \xi_{pj}|\right).
\end{align}
\end{lemma}

  In order to ensure that there exists a feasible solution $\mathcal{I} \subseteq \mathcal{D}$ that satisfies the constraints, we assume that $f_{\theta_p} (\mathcal{D}) \geq 1- R_{\theta_p}$ for all $\theta_p \in \Theta$.  For any $\mathcal{I} \subseteq \mathcal{D}$, we define
     $f_{\theta_p}^{'}(\mathcal{I}) = \text{min} \{ f_{\theta_p}(\mathcal{I}), 1- R_{\theta_p}\} \quad \forall \theta_p \in \Theta.$
  The function $f_{\theta_p}^{'}(\mathcal{I})$ captures the sufficient condition for satisfying the penalty constraints corresponding to each state. We now define, for all $\mathcal{I} \subseteq \mathcal{D}$, 
\begin{equation}
\label{eq:zofn}
   z(\mathcal{I}) = \sum_{\theta_p \in \Theta} f_{\theta_p}^{\prime}(\mathcal{I})=\sum_{\theta_p \in \Theta} \min \left\{f_{\theta_p}(\mathcal{I}), 1-R_{\theta_p}\right\}.
\end{equation}
The expression $z(\mathcal{I})$ combines all the constraints (corresponding to each hypothesis $\theta \in \Theta$), which we wish to satisfy, while selecting the information set. In other words, $z(\cdot)$ is used to find the optimal set $\mathcal{I}$, i.e., a set $\mathcal{I} \subseteq \mathcal{D}$ with minimal $c(\mathcal{I})$,  satisfying $z(\mathcal{I}) = z(\mathcal{D})$. Since $F_{\theta_p}(\emptyset) = \Theta$, we have $ z(\emptyset) = m - \sum_{\theta_p \in \Theta} \max_{\theta_i \in \Theta} \xi_{pi}$. Since $f_{\theta_p}(\mathcal{I})$ is approximately submodular and non-decreasing, we have that $f'_{\theta_p}(\mathcal{I})$ is also approximately submodular and non-decreasing. Noting that the non-negative sum of approximately submodular functions is approximately submodular (Lemma 3.12 of \cite{borodin2014weakly}), we have that $z(\cdot)$ is also approximately submodular. We have the following result, which follows from the existence of a feasible solution for Problem \ref{prob:dsrc}.
\vspace{-5pt}
 \begin{lemma}
 \label{lma2}
     For any $\mathcal{I} \subseteq\mathcal{D}$, the constraint $1-\max_{\theta_i \in F_{\theta_p}(\mathcal{I})} \xi_{pi} \geq 1-R_{\theta_p}$ holds for all $\theta_p \in \Theta$ if and only if $\sum_{\theta_p \in \Theta} f_{\theta_p}^{\prime}(\mathcal{I})=\sum_{\theta_p \in \Theta} f_{\theta_p}^{\prime}(\mathcal{D})$. 
 \end{lemma} 


\noindent We now have from Lemma \ref{lma2} that the constraint (\ref{eq:dsrc}) in Problem \ref{prob:dsrc} can be equivalently written as 

\begin{equation}
\label{eq:submod_opt}
    \begin{aligned}
 \hspace{18pt} \min _{\mathcal{I} \subseteq\mathcal{D}} c(\mathcal{I}); \hspace{2pt} \textit { s.t. } z(\mathcal{I})=z(\mathcal{D}).
\end{aligned}
\end{equation}

Problem (\ref{eq:submod_opt}) can then be viewed as the set covering problem studied in \cite{wolsey1982analysis}. In \cite{das2018approximate}, the authors present performance guarantees for the weak submodular version of the set covering problem studied in \cite{wolsey1982analysis}. Based on Theorem 9 in \cite{das2018approximate}, we have the following performance guarantees for Algorithm \ref{alg:greedy} when applied to the MCIS problem.
\begin{algorithm}[!t]
\caption{Greedy Algorithm for MCIS}\label{alg:greedy}
\textbf{Input:} ${\mathcal{D}}, z:2^{{\mathcal{D}}}\to\mathbb{R}_{\geq 0}, c_i \in \mathbb{R}_{>0}$ $\forall i\in{\mathcal{D}}$ \\
\textbf{Output:} $\mathcal{I}_g$ 
\begin{algorithmic}[1]
\State {$k\gets 0, \mathcal{I}_g^0\gets\emptyset$} 
\While{$z(\mathcal{I}_g^t)< z({\mathcal{D}})$}
    \State {$j_t\in\mathop{\arg\max}_{i\in{\mathcal{D}}\setminus\mathcal{I}_g^t}\frac{z(\mathcal{I}_g^t\cup\{i\})-z(\mathcal{I}_g^t)}{c_i}$} 
    \State {$\mathcal{I}_g^{t+1}\gets \mathcal{I}_g^t\cup\{j_t\}, k\gets k+1$}
\EndWhile
\State {$T\gets k$, $\mathcal{I}_g\gets\mathcal{I}_g^T$}
\State {\textbf{return} $\mathcal{I}_g$}
\end{algorithmic}
\end{algorithm}
\vspace{-5pt}
\begin{theorem}
\label{thm:greedy}
    Let $\mathcal{I}^*$ be an optimal solution to the MCIS problem having a submodularity ratio $\gamma$. For a specified threshold $\mu_{th} \in (0,1)$ and $0\leq\delta\leq1$, with probability at least $1-\delta$, Algorithm 1 under $\Tilde{N}$ observation samples  returns a solution $\mathcal{I}_g$ to the MCIS problem (i.e., (\ref{eq:dsrc})) that satisfies the following:
    $$c\left(\mathcal{I}_g\right) \leq\left(1+\frac{1}{\gamma}\log \frac{z(\mathcal{D})-z(\emptyset)}{z(\mathcal{D})-z\left(\mathcal{I}_g^{T-1}\right)}\right) c\left(\mathcal{I}^*\right),$$
    where $\Tilde{N}$ is specified in (\ref{eq:nbar}), and $\mathcal{I}_g^1, \ldots, \mathcal{I}_g^{T-1}$ are specified in Algorithm \ref{alg:greedy}.
\end{theorem}
We have the following result characterizing the asymptotic performance of the greedy algorithm.
\begin{corollary}
\label{coro:mcis}
    Instate the hypothesis and notation of Theorem \ref{thm2}. As $t \to \infty$, we have the following: (a) $\mu_{\infty}^{\mathcal{I}}(\theta_q) = 0 \quad \forall \theta_q \notin F_{\theta_p}(\mathcal{I})$, and  (b) $\mu_{\infty}^{\mathcal{I}}(\theta_q) = \frac{1}{|F_{\theta_p}(\mathcal{I})|} \quad \forall \theta_q \in F_{\theta_p}(\mathcal{I})$. The near-optimal guarantees provided in Theorem \ref{thm:greedy} for Problem \ref{prob:dsrc} hold with probability 1 (a.s.).  
\end{corollary} 

\section{Minimum-Penalty Information Set Selection}
In this section, we consider the problem where the central node has a fixed budget for selecting information sources and seeks to minimize the maximum penalty of misclassifying the true state. Since the true state is not known a priori, the central designer has to minimize the maximum penalty for each possible true state, which is a multi-objective optimization problem under a budget constraint. We scalarize the multi-objective optimization into a single-objective optimization problem. The optimal solution to this single-objective problem is a Pareto optimal solution to the multi-objective problem (\cite{hwang2012multiple}). We now formalize the Minimum-Penalty Information Set Selection (MPIS) Problem as follows.

\begin{Problem}[MPIS]
\label{prob:mpis}
    Consider a set $\Theta=\left\{\theta_1, \ldots, \theta_m\right\}$ of possible states of the world; a set $\mathcal{D}$ of information sources, with each source $i \in \mathcal{D}$ having a cost $c_i \in \mathbb{R}_{\geq0}$; a row-stochastic penalty matrix $ \Xi = [\xi_{ij}] \in \mathbb{R}^{m \times m}$ ; and a selection budget $K \in \mathbb{R}_{\geq0}$. The MPIS Problem is to find a set of selected information sources $\mathcal{I} \subseteq \mathcal{D}$ that solves
\begin{equation}
\label{eq:mpis}
\begin{aligned}
 \min _{\mathcal{I} \subseteq \mathcal{D}} & \hspace{5pt} \displaystyle \sum_{\theta_p \in \Theta} \left(\max_{\theta_j \in F_{\theta_p}(\mathcal{I})} \xi_{pj} \right) ;
  \hspace{2pt} \text { s.t. } \sum_{i \in \mathcal{I}} c_i \leq K.
\end{aligned}
\end{equation}
\end{Problem}
Consider the following equivalent optimization problem: 
\begin{equation}
\label{eq:mpis_max}
\begin{aligned}
 \max _{\mathcal{I} \subseteq \mathcal{D}} & \hspace{5pt} \displaystyle \sum_{\theta_p \in \Theta} \left(1-\max_{\theta_j \in F_{\theta_p}(\mathcal{I})} \xi_{pj} \right);
  \hspace{2pt} \textit { s.t. } \sum_{i \in \mathcal{I}} c_i \leq K.
\end{aligned}
\end{equation}
It is easy to verify that the problem defined in (\ref{eq:mpis_max}) is equivalent to the problem defined in (\ref{eq:mpis}), i.e., the information set $\mathcal{I} \subseteq \mathcal{D}$ that optimizes the problem in Equation (\ref{eq:mpis_max}) is also the optimal solution to the Problem \ref{prob:mpis}. 
We note that $f_{\theta_p}(\mathcal{I}) = 1-\max_{\theta_j \in F_{\theta_p}(\mathcal{I})} \xi_{pj}$. We denote $\Lambda(\mathcal{I}) = \sum_{\theta_p \in \Theta} f_{\theta_p}(\mathcal{I})$. 

From Lemma \ref{lma:submod} and Lemma 3.12 in \cite{borodin2014weakly}, it follows that the objective function in (\ref{eq:mpis_max}) is approximately submodular with the submodularity ratio $\gamma$. Based on the guarantees for greedy maximization of monotone, non-decreasing, approximately submodular functions subject to Knapsack constraints in Theorem 6 of \cite{das2018approximate}, we have the following result. 
\begin{algorithm}[!t]
\caption{Greedy Algorithm for MPIS}\label{alg:greedy2}
\textbf{Input:} $\text{Data sources: }\mathcal{D},\text{Penalties: }  \Xi \in \mathbb{R}^{m \times m}, \text{Selection costs: }  c_i \hspace{2pt} \forall i \in \mathcal{D} , \text{Budget: } K \in \mathbb{R}_{>0}$  \\
\textbf{Output:} $\mathcal{I}_K$ 
\begin{algorithmic}[1]
\State {$t\gets 0, \mathcal{I}_K\gets\emptyset$} 
\While{$t \leq K$}
    \State {$j\gets\mathop{\arg\max}_{i\in{\mathcal{D}}\setminus\mathcal{I}_K} \frac{ \Lambda(\mathcal{I}_K \cup \{i\}) - \Lambda(\mathcal{I}_K)}{c_i}$}
    \State {$\mathcal{I}_K\gets \mathcal{I}_K\cup\{j\}, t\gets t+c_j$}
\EndWhile
\State {\textbf{return} $\mathcal{I}_K$}
\end{algorithmic}
\end{algorithm}

 \begin{theorem}
 \label{thm:greedy2}
     Let $\mathcal{I}_K \subseteq \mathcal{D}$ denote the information set selected by Algorithm \ref{alg:greedy2} and let $\mathcal{I}^*_K \subseteq \mathcal{D}$ denote the optimal information set for the MPIS Problem with submodularity ratio $\gamma$. For a specified threshold $\mu_{th} \in (0,1)$ and $0\leq\delta\leq1$, with probability at least $1-\delta$, Algorithm \ref{alg:greedy2} under $\Tilde{N}$ observation samples  returns a solution $\mathcal{I}_K$ to the MPIS problem (i.e., (\ref{eq:mpis})) that satisfies 
         $\Lambda(\mathcal{I}_K)  \geq \left(1 - e^{-\gamma} \right) \Lambda(\mathcal{I}^*_K) + c$ ,
     where $c = \Lambda(\emptyset)/e^{\gamma}$ and $\Tilde{N}$ is specified in (\ref{eq:nbar}).
 \end{theorem}
 We have the following result characterizing the asymptotic performance of the greedy algorithm.
\begin{corollary}
\label{coro:mpis}
    Instate the hypothesis and notation of Theorem \ref{thm2}. As $t \to \infty$, we have the following: (a) $\mu_{\infty}^{\mathcal{I}}(\theta_q)= 0 \quad \forall \theta_q \notin F_{\theta_p}(\mathcal{I})$, and  (b) $\mu_{\infty}^{\mathcal{I}}(\theta_q) = \frac{1}{|F_{\theta_p}(\mathcal{I})|} \quad \forall \theta_q \in F_{\theta_p}(\mathcal{I})$. The near-optimal guarantees provided in Theorem \ref{thm:greedy2} for Problem \ref{prob:mpis} hold with probability 1 (a.s.).  
\end{corollary} 
\section{Alternate Penalty Metric for Information Set Selection}
\label{sec:alt_form}
In many practical scenarios, the submodularity ratio of the maximum penalty metric may be arbitrarily small (or zero) when misclassification penalties for two hypotheses are very close to each other (or equal) (see Appendix \ref{app:limit} for a detailed discussion). It is also easy to verify that the submodularity ratio $\gamma$ decreases as the number of hypotheses increase. As a result, the performance bounds for the greedy algorithms become weaker. In such scenarios, one can turn to an alternate metric for optimization, which can provide non-trivial guarantees for the performance of the greedy algorithm. To this end, we present an alternate metric to characterize the quality of an information set, based on the total penalty of misclassification, defined as follows: 
\begin{equation}
    \label{eq:tot_pen}
    \rho_{\theta_p}(\mathcal{I}) = \sum_{\theta_i \in F_{\theta_p}(\mathcal{I})} \xi_{pi}.
\end{equation}
Intuitively, in order to minimize the total penalty ($\rho_{\theta_p}(\mathcal{I})$) (or ensure that it is below a desired bound), one has to select a subset  $ \mathcal{I} \subseteq \mathcal{D}$ that ensures that the number of hypotheses which are observationally equivalent to the true hypothesis $\theta_p$, i.e., $|F_{\theta_p}(\mathcal{I})|$, is small and/or the hypotheses that are observationally equivalent to the true hypothesis have lower misclassification penalties. Effectively, this results in lower penalty associated with misclassifying the true hypothesis.

We define the Modified Minimum Cost Information Set Selection (M-MCIS) and Modified Minimum Penalty Information Set Selection (M-MPIS) Problems based on this metric as follows. 
\begin{Problem}[M-MCIS]
\label{prob:mmcis}
    Consider a set $\Theta=\left\{\theta_1, \ldots, \theta_m\right\}$ of possible states of the world, a set $\mathcal{D}$ of information sources, a selection cost $c_i \in \mathbb{R}_{>0}$ of each source $i \in \mathcal{D}$, a row-stochastic penalty matrix $\Xi = [\xi_{ij}] \in \mathbb{R}^{m \times m}$, and prescribed penalty bounds $0 \leq R'_{\theta_p} \leq 1$ for all $\theta_p \in \Theta$. The M-MCIS Problem is to find a set of selected information sources $\mathcal{I} \subseteq \mathcal{D}$ that solves
\begin{equation}
\label{eq:mmcis}
\begin{aligned}
 \min _{\mathcal{I} \subseteq \mathcal{D}} & \hspace{5pt} c(\mathcal{I}) ; \hspace{2pt}
 \text { s.t. } \displaystyle \rho_{\theta_p} (\mathcal{I})  \leq R'_{\theta_p} \quad \forall \theta_p \in \Theta.
\end{aligned}
\end{equation}
\end{Problem}
Note that the penalty bounds $R_{\theta_p}$ of the MCIS Problem differ from the bounds $R'_{\theta_p}$ of M-MCIS Problem, as the former is a bound on the maximum penalty, while the latter is a bound on the total penalty. The designer can choose the bounds $R'_{\theta_p}$ in order to achieve the desired classification performance. 
\begin{Problem}[M-MPIS]
\label{prob:mmpis}
    Consider a set $\Theta=\left\{\theta_1, \ldots, \theta_m\right\}$ of possible states of the world; a set $\mathcal{D}$ of information sources, with each source $i \in \mathcal{D}$ having a cost $c_i \in \mathbb{R}_{\geq0}$; a row-stochastic penalty matrix $ \Xi = [\xi_{ij}] \in \mathbb{R}^{m \times m}$ ; and a selection budget $K \in \mathbb{R}_{\geq0}$. The M-MPIS Problem is to find a set of selected information sources $\mathcal{I} \subseteq \mathcal{D}$ that solves
\begin{equation}
\label{eq:mmpis}
\begin{aligned}
 \min _{\mathcal{I} \subseteq \mathcal{D}} & \hspace{5pt} \displaystyle \sum_{\theta_p \in \Theta} \rho_{\theta_p}(\mathcal{I}) ;
  \hspace{2pt} \text { s.t. } \sum_{i \in \mathcal{I}} c_i \leq K.
\end{aligned}
\end{equation}
\end{Problem}

\begin{lemma}
    \label{lma:surr}
    The function $g_{\theta_p}(\mathcal{I})= 1 -\rho_{\theta_p}(\mathcal{I}):2^{\mathcal{D}}\to \mathbb{R}_{\ge 0}$ is submodular for all $\theta_p \in \Theta$.
\end{lemma}
By Lemma \ref{lma:surr}, we have the following result characterizing the performance of the greedy algorithms for the modified information set selection problems.

\begin{corollary}
    \label{coro:submod_greedy}
For Algorithm \ref{alg:greedy} (respectively Algorithm \ref{alg:greedy2}) applied to M-MCIS (respectively M-MPIS) Problem, the near-optimal guarantees provided in Theorem \ref{thm:greedy} (respectively Theorem \ref{thm:greedy2}) hold with $\gamma=1$. 
\end{corollary}
From Corollary \ref{coro:submod_greedy}, we have that the total penalty metric enjoys stronger near-optimal guarantees (due to submodularity) compared to that of the maximum penalty metric (which is weak submodular) for greedy optimization. Moreover, the near-optimal guarantees for the M-MCIS and M-MPIS problems are independent of the misclassification penalties and the number of hypotheses.

\section{Empirical Evaluation}
In this section, we validate the theoretical results through numerical simulations. We present simulations for varying submodularity ratios, finite sample convergence of the beliefs and the modified information selection problems in Appendix \ref{app:exp}.

 We consider a hypothesis testing task where one has to identify (or classify) an aerial vehicle into one of the following $10$ classes: $\Theta =$ \textit{ \{cargo, passenger, freight, heavy fighter, interceptor, sailplane, hang glider, paraglider, surveillance UAV, quadrotor\}}. We will refer to this as the Aerial Vehicle Classification task (AVC task). The penalty matrix is as shown in Figure \ref{fig:example} (a). Each row of the penalty matrix is normalized.  We set $|\mathcal{D}| = 10$, the costs $c_i$ for $i \in \mathcal{D}$ are sampled uniformly from $\{1,\hdots, 10\}$. We consider the infinite-observation case and randomly generate the observationally equivalent sets $F_{\theta_p}(i)$ for each $\theta_p \in \Theta$ and $i \in \mathcal{D}$. We first consider the minimum cost information set selection problem for the AVC task. The thresholds $R_{\theta_p}$ for $\theta_p \in \{ \textit{cargo, passenger, freight, sailplane, hang glider, paraglider}\}$ are randomly sampled from $[0.7,1]$ and for $\theta_p \in \{ \textit{heavy fighter, interceptor, surveillance UAV, quadrotor}\}$ are randomly sampled from $[0.1,0.4]$.
For 100 randomly generated instances, we run Algorithm \ref{alg:greedy} to find the greedy information set $\mathcal{I}_g$ and find the optimal information set $\mathcal{I}^*$ using brute-force search. We plot the ratio of cost of the greedy information set to that of the optimal, i.e.,   $c(\mathcal{I}_g) / c(\mathcal{I^*}) $, in Figure \ref{fig:example} (b).
\vspace{-5pt}
\begin{figure}[!ht]
    \centering
    \subfigure[]{\includegraphics[width=0.3\textwidth]{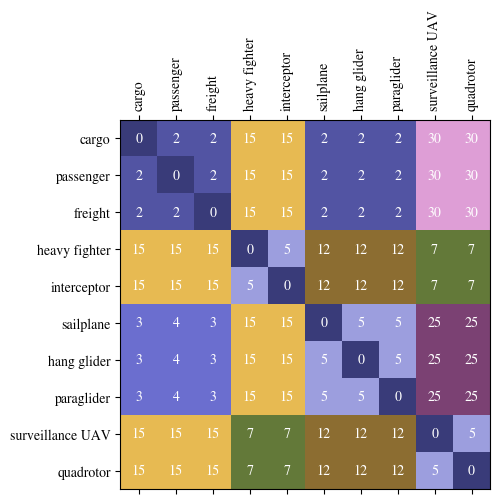}} 
    \subfigure[]{\includegraphics[width=0.31\textwidth]{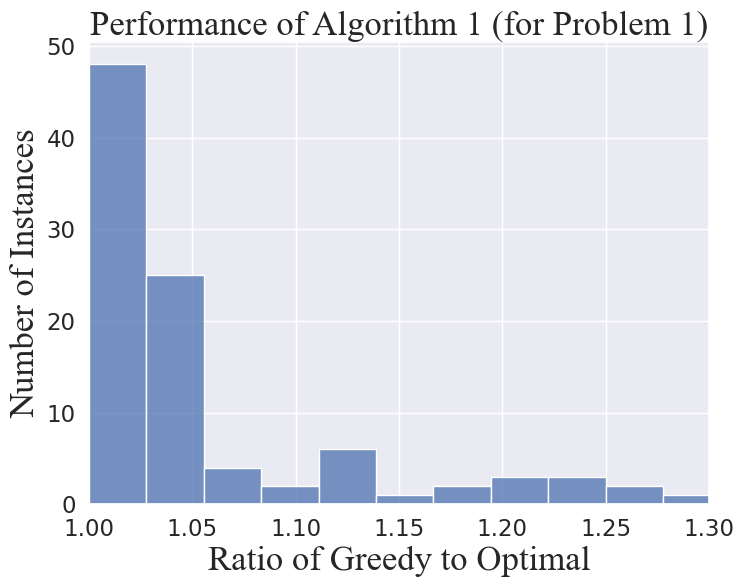}} 
    \subfigure[]{\includegraphics[width=0.3\textwidth]{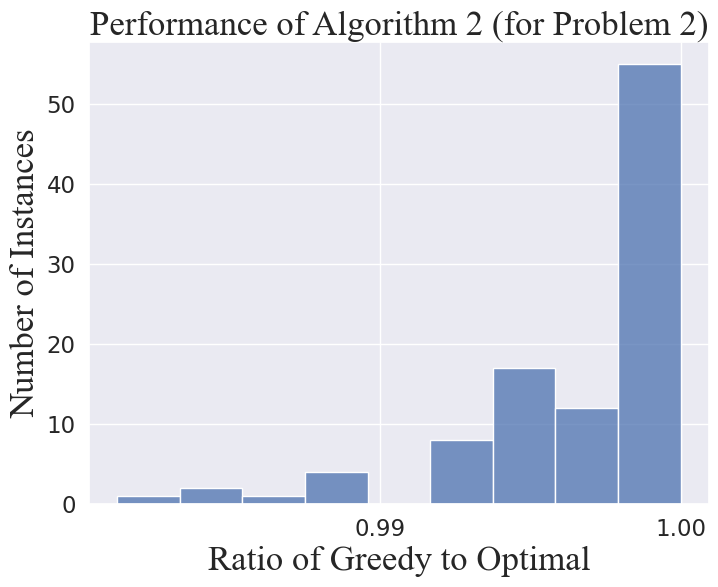}}
    \caption{ (a) Penalty Matrix for the Aerial Vehicle Classification (AVC) task, (b) Performance of Algorithm \ref{alg:greedy} (for Problem \ref{prob:dsrc}), (c) Performance of Algorithm \ref{alg:greedy2} (for Problem \ref{prob:mpis}).}
    \label{fig:example}
\end{figure}

Next, we consider the minimum penalty information set selection problem for the AVC task. We generate 100 random instances with varying information source costs and selection budgets. We run Algorithm \ref{alg:greedy2} to find the greedy information set $\mathcal{I}_g$ and find the optimal information set $\mathcal{I}^*$ using brute-force search. We plot the ratio of greedy utility to that of the optimal, i.e., $\Lambda(\mathcal{I}_g)/\Lambda(\mathcal{I^*})$, in Figure \ref{fig:example}(c). These plots show the near-optimal performance of the greedy algorithm. Note that the penalty matrix for these instances does not satisfy Assumption 2 (uniqueness of misclassification penalties). Thus, these problem instances are not guaranteed to exhibit the weak submodularity property. Despite this, we observe that the greedy algorithms provide near-optimal performance.

\section{Conclusion}
In this work, we studied two variants of an information set selection problem for hypothesis testing: (i) selecting a minimum cost information set to ensure the maximum penalty for misclassifying the true hypothesis is below a desired bound and (ii) optimal information set selection under a limited budget to minimize the maximum penalty of misclassifying the true hypothesis. Leveraging the weak submodularity property of the performance metric, we established high-probability guarantees for greedy algorithms for both problems, along with the associated finite sample convergence rates for the Bayesian beliefs. Next, we proposed an alternate metric based on the total penalty of misclassification for information set selection, which enjoys (stronger) near-optimal performance guarantees with high-probability for the greedy algorithms. Finally, we evaluated the empirical performance of the proposed greedy algorithms over several randomly generated problem instances.
\newpage
 \acks{This material is based upon work supported by the Office of Naval Research (ONR) and Saab, Inc. under the Threat and Situational Understanding of Networked Online Machine Intelligence (TSUNOMI) program (grant no. N00014-23-C-1016). Any opinions, findings, and conclusions or recommendations expressed in this material are those of the author(s) and do not necessarily reflect the views of the ONR and/or Saab, Inc.}
 \bibliography{ref.bib}

 \newpage

\appendix
\section{Proofs}
\label{app:proof}
\subsection{{Proof of Theorem \ref{thm2}}}
\begin{theorem*}
    Let the true state of the world be $\theta_p$ and let $\mu_0 (\theta) = \frac{1}{m} \hspace{5pt} \forall \theta \in \Theta$ (i.e., uniform prior). Under Assumption 1, for $\delta, \epsilon \in [0,1]$, and L as defined in Equation (\ref{eq:kld_ratio}), and for an information set $\mathcal{I} \subseteq \mathcal{D}$, the Bayesian update rule in Equation (\ref{eq:bayes_full}) has the following property: there is an integer $N(\delta, \epsilon, L)$, such that with probability at least $1-\delta$, for all $t > N(\delta, \epsilon, L)$ we have the following: 
    \begin{enumerate}[label=(\alph*)]
        \item  $\mu_t^{\mathcal{I}}(\theta_q) =  \mu_t^{\mathcal{I}}(\theta_p)  \hspace{5pt} \forall \theta_q \in F_{\theta_p}(\mathcal{I})$, and
        \item  $\mu_t^{\mathcal{I}}(\theta_q) \leq \exp{(-t(|K(\theta_p, \theta_q)-\epsilon|))} \hspace{5pt} \forall \theta_q \notin F_{\theta_p}(\mathcal{I})$;
    \end{enumerate}
    where $K(\theta_p, \theta_q) = D_{KL}(\ell_{\mathcal{I}}(\cdot | \theta_p) || \ell_{\mathcal{I}}(\cdot | \theta_q) )$ is the Kullback-Leibler divergence measure between the likelihood functions $\ell_{\mathcal{I}}(\cdot | \theta_p)$ and $ \ell_{\mathcal{I}}(\cdot | \theta_q) $, $F_{\theta_p}(\mathcal{I})$ is defined in (\ref{eq:obs_eq_set}), and 
    \begin{equation}
        N(\delta, \epsilon,L) = \left\lceil \frac{2L^2}{\epsilon^2 } \log \frac{2}{\delta} \right\rceil.
    \end{equation}
\end{theorem*}
\begin{proof} The arguments presented in this proof are similar to Lemma 1 in \cite{mitra2020new}. Consider a class $\theta_q = \Theta \setminus \theta_p$, which is not the true state of the world. We now define, for all $\theta_q \in \Theta \setminus \theta_p$ and for all $k \in \mathbb{N}_+$,  the following: 
\begin{equation}
    \label{eq:alpha}
    \phi_k^{\mathcal{I}}(\theta_q) = \log \frac{\mu_k^{\mathcal{I}}(\theta_q)}{\mu_k^{\mathcal{I}}(\theta_p)} \hspace{4pt}\text{and}\hspace{4pt} \eta_k^{\mathcal{I}}(\theta_q) = \log \frac{\ell_{\mathcal{I}}(o_k^{\mathcal{I}}|\theta_q)}{\ell_{\mathcal{I}}(o_k^{\mathcal{I}}|\theta_p)},
\end{equation}
    where $o_k^{\mathcal{I}} \in \mathcal{O}_{\mathcal{I}}$ is the joint observation  and $\ell_{\mathcal{I}}(\cdot)$ is the joint likelihood function of the information set $\mathcal{I} \subseteq \mathcal{D}$. From the Bayesian update rule in Equation (\ref{eq:bayes_full}), we have:

    \begin{equation}
        \label{eq:alpha_beta}
        \phi_{t}^{\mathcal{I}}(\theta_q) = \phi_0^{\mathcal{I}}(\theta_q) + \sum_{k=1}^{t}\eta_k^{\mathcal{I}}(\theta_q).
    \end{equation}
    Since we assume a uniform prior, i.e., $\mu_0 (\theta) = \frac{1}{m} \hspace{5pt} \forall \theta \in \Theta$, we have that $\phi_0^{\mathcal{I}}(\theta_q) =0 \hspace{4pt} \forall \theta_q \in \Theta \setminus \theta_p$ and thus $ \phi_{t}^{\mathcal{I}}(\theta_q) =  \sum_{k=1}^{t}\eta_k^{\mathcal{I}}(\theta_q), \hspace{4pt} \forall t \in \mathbb{N}_+$. Note that $\{ \eta_k^{\mathcal{I}}(\theta_q) \}$ is a sequence of \textit{i.i.d.} random variables which are bounded and have a finite mean (by Equation \ref{eq:kld_ratio}). Each random variable $\eta_k^{\mathcal{I}}(\theta_q)$ has a mean given by $-K(\theta_p, \theta_q)$, where the mean is obtained by using the expectation operator $\mathbb{E}^{\theta_p}[\cdot]$ associated with the probability measure $\mathbb{P}^{\theta_p}$ as defined in Section \ref{sec:mcis}. By the strong law of large numbers, we have that $\frac{1}{t} \sum_{k=1}^{t}\eta_k^{\mathcal{I}}(\theta_q) \to -K(\theta_p, \theta_q) $ asymptotically almost surely (a.a.s.). 
    
    If $\theta_q \in F_{\theta_p}(\mathcal{I})$, we know from (\ref{eq:obs_eq_set}) that $\eta_k^{\mathcal{I}}(\theta_q) = 0$ for all $k \in \{1,\hdots, t\}$, and thus we have $\phi_{t}^{\mathcal{I}}(\theta_q) = 0$ for all $t \in \mathbb{N}_+$. This directly implies that $\mu_t^{\mathcal{I}}(\theta_q) =  \mu_t^{\mathcal{I}}(\theta_p)  \hspace{5pt} \forall \theta_q \in F_{\theta_p}(\mathcal{I})$, establishing part (a) of the result. Let $\hat{K}(\theta_p, \theta_q) = - \frac{1}{t} \sum_{k=1}^{t}\eta_k^{\mathcal{I}}(\theta_q)$ denote the sample mean (estimated KL divergence). Now consider a state $\theta_q \notin F_{\theta_p}(\mathcal{I})$, Equation (\ref{eq:alpha_beta}) can be equivalently written as
    \begin{equation}
    \label{eq:bel_eq}
        \mu_t^{\mathcal{I}}(\theta_q) = \mu_t^{\mathcal{I}}(\theta_p) \exp{\left( t. \frac{1}{t}\sum_{k=1}^{t}\eta_k^{\mathcal{I}}(\theta_q)\right)} = \mu_t^{\mathcal{I}}(\theta_p) \exp{( -t \hat{K}(\theta_p,\theta_q) )}.
    \end{equation}
     By Equation (\ref{eq:kld_ratio}) and Hoeffding's Inequality (\cite{hoeffding1994probability}), for all $\epsilon>0$, we have the following: 
    \begin{equation}
        \mathbb{P} \left( \left| \frac{1}{t} \sum_{k=1}^{t}\eta_k^{\mathcal{I}}(\theta_q) - (-K(\theta_p, \theta_q)) \right| \geq \epsilon  \right) \leq 2 \exp{\left( -\frac{ \epsilon^2 t}{2L^2} \right)}.
    \end{equation}
    This condition is equivalent to: 
       \begin{equation}
       \label{eq:hoeffding}
        \mathbb{P} \left( \left| \frac{1}{t} \sum_{k=1}^{t}\eta_k^{\mathcal{I}}(\theta_q) - (-K(\theta_p, \theta_q)) \right| \leq \epsilon  \right) \geq 1- 2\exp{\left( -\frac{ \epsilon^2 t}{2L^2} \right)}.
    \end{equation}
    Now let $\delta = 2\exp{\left( -\frac{ \epsilon^2 t}{2L^2}\right)}$ which yields $t =  \frac{2L^2}{\epsilon^2 } \log \frac{2}{\delta}$. The condition in Equation (\ref{eq:hoeffding}) means that, with probability at least $1-\delta$, we have that $\left|-\hat{K}(\theta_p, \theta_q)  - (-K(\theta_p, \theta_q)) \right| \leq \epsilon$, for all $t \geq \frac{2L^2}{\epsilon^2 } \log \frac{2}{\delta}$. We know that $\mu_t^{\mathcal{I}}(\theta_p) \leq 1$ for any $t \in \mathbb{N}_+$. Now, by combining Equations (\ref{eq:bel_eq}) and (\ref{eq:hoeffding}), with probability at least $1-\delta$, for all $t > N(\delta,\epsilon,L)$, we have
    \begin{equation}
        \label{eq:bel_bound}
        \mu_t^{\mathcal{I}}(\theta_q) \leq \exp{(-t(|K(\theta_p, \theta_q)-\epsilon|))} \hspace{5pt} \quad \forall \theta_q \notin F_{\theta_p}(\mathcal{I}),
    \end{equation}
    where $N(\delta,\epsilon,L) = \left\lceil \frac{2L^2}{\epsilon^2 } \log \frac{2}{\delta} \right\rceil$, establishing part (b) of the result.
    \end{proof}

\subsection{Proof of Corollary \ref{coro1}}
\begin{corollary*}
    Instate the hypothesis and notation of Theorem \ref{thm2}. For a specified threshold $\mu_{th}\in    (0,1)$ for the belief over any class $\theta_q \notin F_{\theta_p}(\mathcal{I}) $, there exists $\delta, \epsilon \in [0,1]$, for which one can guarantee with probability at least $1-\delta$ that $\mu_t^{\mathcal{I}}(\theta_q) \leq \mu_{th}$  for all $\theta_q \notin F_{\theta_p}$ and for all $t > \Tilde{N}$, where 
    \begin{equation}
         \Tilde{N} = \left\lceil  \max \left\{ \frac{2L^2}{\epsilon^2 } \log \frac{2}{\delta} , \frac{1}{ \displaystyle\min_{\theta_p,\theta_q \in \Theta} |K(\theta_p,\theta_q) -\epsilon| } \log \frac{1}{\mu_{th}}   \right\} \right\rceil.
    \end{equation}
\end{corollary*}
\begin{proof}
    From Theorem \ref{thm2}, we know that with probability at least $1-\delta$, for all $t > N(\delta,\epsilon,L)$, we have
    \begin{equation}
        \mu_t^{\mathcal{I}}(\theta_q) \leq \exp{(-t(|K(\theta_p, \theta_q)-\epsilon|))} \hspace{5pt} \quad \forall \theta_q \notin F_{\theta_p}(\mathcal{I}),
    \end{equation}
    where $N(\delta,\epsilon,L) = \left\lceil \frac{2L^2}{\epsilon^2 } \log \frac{2}{\delta} \right\rceil$. Since we require $\mu_t^{\mathcal{I}}(\theta_q) \leq \mu_{th}$  for all $\theta_q \notin F_{\theta_p}$, we let 
        \begin{equation}
        \exp{(-t(|K(\theta_p, \theta_q)-\epsilon|))}  <  \mu_{th}\hspace{5pt} \quad \forall \theta_q \notin F_{\theta_p}(\mathcal{I}) .
    \end{equation}
    Re-arranging the terms in the above equation, we get
      \begin{equation}
        t > \frac{1}{|K(\theta_p, \theta_q)-\epsilon|} \log \frac{1}{\mu_{th}}
    \end{equation}
    In order to ensure $\mu_t^{\mathcal{I}}(\theta_q) \leq \mu_{th} \quad \forall \theta_q \notin F_{\theta_p}(\mathcal{I}) $ is satisfied for any true hypothesis $\theta_p \in \Theta$, we need the following condition to be satisfied.
      \begin{equation}
      \label{eq:ts}
        t >   \frac{1}{\displaystyle\min_{\theta_p, \theta_q \in \Theta}|K(\theta_p, \theta_q)-\epsilon|} \log \frac{1}{\mu_{th}}
    \end{equation}
    In other words, the number of samples required for the beliefs over the hypotheses $\theta_q \notin F_{\theta_p}(\mathcal{I})$ to be bounded under the specified threshold $\mu_{th}$ depends on the least KL divergence measure between the likelihood functions of any $\theta_q \notin F_{\theta_p}(\mathcal{I})$ and $\theta_p$, over all possible true hypotheses $\theta_p \in \Theta$.
    From Equation (\ref{eq:ts}) and the fact that $t>N(\delta,\epsilon,L)$, we obtain $\Tilde{N}$.
\end{proof}
\subsection{Proof of Lemma \ref{lma:submod}}
\begin{lemma*} Under Assumption 2, the function $f_{\theta_p}(\mathcal{I}) :2^{\mathcal{D} }\to\mathbb{R}_{\geq0} $ is approximately submodular for all $\theta_p \in \Theta$, with a submodularity ratio $\gamma = \underline{\xi}/ \Bar{\xi} $, where 
\begin{align}
\label{eq:xifloor_1}
\displaystyle \underline{\xi} &= \min_{\theta_p \in \Theta} \left(\min_{\theta_i,\theta_j \in \Theta}  |\xi_{pi} - \xi_{pj}|\right); \\
\label{eq:xibar}
    \Bar{\xi} &= \max_{\theta_p \in \Theta} \left(\max_{\theta_i,\theta_j \in \Theta}  |\xi_{pi} - \xi_{pj}|\right).
\end{align}
\end{lemma*}
\begin{proof}
Recall that $f_{\theta_p}(\mathcal{I}) =1- \max_{\theta_i\in F_{\theta_p}(\mathcal{I})} \xi_{pi}.$  We begin by proving the following statement:
\begin{equation}
    \sum_{a \in A \setminus B} (\ef(\{a\} \cup B) - \ef(B)) = 0 \implies  \ef(A \cup B) - \ef(B) = 0.
\end{equation}
Now let $\sum_{a \in A \setminus B} (\ef(\{a\} \cup B) - \ef(B)) = 0.$ This implies 
\begin{align} 
    \ef(\{a\} \cup B) - \ef(B) &= 0 \quad \forall a \in A \setminus B \\
    \max_{\theta_i \in F_{\theta_p}(B)} \xi_{pi} - \max_{\theta_i \in F_{\theta_p}(\{a\} \cup B)} \xi_{pi} &= 0 \quad \forall a \in A \setminus B \\
     \max_{\theta_i \in F_{\theta_p}(B)} \xi_{pi} - \max_{\theta_i \in F_{\theta_p}(a) \cap F_{\theta_p}( B)} \xi_{pi} &= 0 \quad \forall a \in A \setminus B \\ \label{eq:lhs_bound}
    \max_{\theta_i \in F_{\theta_p}(B)}  \xi_{pi} = \max_{\theta_i \in F_{\theta_p}(a) \cap F_{\theta_p}( B)} \xi_{pi}   & \quad \quad \forall a \in A \setminus B.
\end{align}
Let $\displaystyle\max_{ \theta_i \in F_{\theta_p}(B)} \xi_{pi} = \xi_{pq}$ for $\theta_q \in F_{\theta_p}(B)$. From Equation (\ref{eq:lhs_bound}) and Assumption 2, we have 
\begin{equation}
\displaystyle \max_{\theta_i \in F_{\theta_p}(a) \cap F_{\theta_p}( B)} \xi_{pi} = \xi_{pq} \quad \forall a \in A \setminus B,
\end{equation}
and it follows that $\theta_q \in F_{\theta_p}(a) \cap F_{\theta_p}( B), \forall a \in A \setminus B$. \newline

\noindent We have the following: $\displaystyle \ef(A \cup B) - \ef(B)= \max_{\theta_i \in F_{\theta_p}(B)} \xi_{pi} - \max_{\theta_i \in F_{\theta_p}(A \cup B)} \xi_{pi}$.\\

\noindent Using the fact that $A \cup B = (A \setminus B) \cup B$, we have $F_{\theta_p}(A \cup B) = F_{\theta_p}(A\setminus B) \cap F_{\theta_p}(B).$\\

\noindent Thus, we have
\begin{align}
    \label{eq:rhs_exp}
    \displaystyle \ef(A \cup B) - \ef(B)= & \max_{\theta_i \in F_{\theta_p}(B)} \xi_{pi} - \max_{\theta_i \in F_{\theta_p}(A\setminus B) \cap F_{\theta_p}(B)} \xi_{pi}.
\end{align}
From the previous argument, we have $\theta_q \in F_{\theta_p}(a) \cap F_{\theta_p}( B), \forall a \in A \setminus B$ and $\displaystyle\max_{ \theta_i \in F_{\theta_p}(B)} \xi_{pi} = \xi_{pq}$ for $\theta_q \in F_{\theta_p}(B).$ Since $F_{\theta_p}(A \setminus B) = \displaystyle \bigcap_{a \in A \setminus B} F_{\theta_p}(a)$, we have $$F_{\theta_p}(A\setminus B) \cap F_{\theta_p}(B) = \left( \displaystyle \bigcap_{a \in A \setminus B} F_{\theta_p}(a) \right) \bigcap F_{\theta_p}(B).$$  This implies $\theta_q \in F_{\theta_p}(A\setminus B) \cap F_{\theta_p}( B).$ It directly follows that 
\begin{equation}
 \max_{\theta_i \in F_{\theta_p}(B)} \xi_{pi} = \max_{\theta_i \in F_{\theta_p}(A\setminus B) \cap F_{\theta_p}(B)} \xi_{pi} = \xi_{pq}.
\end{equation}
Thus, we have $\ef(A \cup B) - \ef(B) = 0$. This gives the trivial bound of $\gamma \le 1$ (since we define $0/0=1$ for characterizing $\gamma$ (see \cite{das2018approximate})). Therefore, in order to establish a non-trivial lower bound on $\gamma$, we consider $\sum_{a \in A \setminus B} (\ef(\{a\} \cup B) - \ef(B)) > 0.$\\

\noindent We now proceed to establish a non-trivial lower bound on $\sum_{a \in A \setminus B} (\ef(\{a\} \cup B) - \ef(B)).$
\begin{align}
    \sum_{a \in A \setminus B} (\ef(\{a\} \cup B) - \ef(B)) &= \sum_{a \in A \setminus B} \left(\max_{\theta_i \in F_{\theta_p}(B)} \xi_{pi} - \max_{\theta_i \in F_{\theta_p}(\{a\} \cup B)} \xi_{pi}\right);\\
    & \geq \min_{\theta_i, \theta_j \in \Theta} |\xi_{pi} - \xi_{pj}|.
\end{align}
For all $\theta_p \in \Theta$ and for all $A, B \subseteq \mathcal{D}$, we have 
\begin{equation}
\label{eq:lb}
    \sum_{a \in A \setminus B} (\ef(\{a\} \cup B) - \ef(B)) \ge \min_{\theta_p  \in \Theta} \left( \min_{\theta_i, \theta_j \in \Theta} |\xi_{pi} - \xi_{pj}| \right) = \underline{\xi}. \quad \text{(Equation (\ref{eq:xifloor_1}))}
\end{equation}
Next, we provide an upper bound on $\displaystyle \ef(A \cup B) - \ef(B).$ 
\begin{align}
    \ef(A \cup B) - \ef(B) &=  \max_{\theta_i \in F_{\theta_p}(B)} \xi_{pi} - \max_{\theta_i \in F_{\theta_p}(A \cup B)} \xi_{pi} \\
    & \leq \max_{\theta_i, \theta_j \in \Theta} |\xi_{pi} - \xi_{pj} |.
\end{align}
For all $\theta_p \in \Theta$ and for all $A, B \subseteq \mathcal{D}$, we have
\begin{equation}
    \label{eq:ub}
    \ef(A \cup B) - \ef(B) \le \max_{\theta_p  \in \Theta} \left( \max_{\theta_i, \theta_j \in \Theta} |\xi_{pi} - \xi_{pj}| \right) = \Bar{\xi}. \quad \text{(Equation (\ref{eq:xibar}))}
\end{equation}
Due to Assumption 2, we have $0< \underline{\xi}<1$ and $0< \Bar{\xi} < 1$.
Combining inequalities (\ref{eq:lb}), (\ref{eq:ub}) and (\ref{eq:submod_ratio_def}), we have that the function $f_{\theta_p}(\mathcal{I})$ is approximately submodular with a submodularity ratio $\gamma = \underline{\xi}/\Bar{\xi}$, for all $\theta_p \in \Theta$. 
\end{proof}
\subsection{Proof of Lemma \ref{lma2}}
\begin{lemma*}
     For any $\mathcal{I} \subseteq\mathcal{D}$, the constraint $1-\max_{\theta_i \in F_{\theta_p}(\mathcal{I})} \xi_{pi} \geq 1-R_{\theta_p}$ holds for all $\theta_p \in \Theta$ if and only if $\sum_{\theta_p \in \Theta} f_{\theta_p}^{\prime}(\mathcal{I})=\sum_{\theta_p \in \Theta} f_{\theta_p}^{\prime}(\mathcal{D})$. 
 \end{lemma*} 

\begin{proof} Suppose the constraints $1-\max_{\theta_i \in F_{\theta_p}(\mathcal{I})} \xi_{pi} \geq 1-R_{\theta_p}$ hold for all $\theta_p \in \Theta$. It follows that $f_{\theta_p}^{\prime}(\mathcal{I})= 1-R_{\theta_p}$ for all $\theta_p \in \Theta$. Noting that, $f_{\theta_p}(\mathcal{D}) \geq 1-R_{\theta_p}$, we have $f_{\theta_p}^{\prime}(\mathcal{D})=1-R_{\theta_p}$ for all $\theta_p \in \Theta$, which implies $\sum_{\theta_p \in \Theta} f_{\theta_p}^{\prime}(\mathcal{I})=\sum_{\theta_p \in \Theta} f_{\theta_p}^{\prime}(\mathcal{D})$. Conversely, suppose $\sum_{\theta_p \in \Theta} f_{\theta_p}^{\prime}(\mathcal{I})=\sum_{\theta_p \in \Theta} f_{\theta_p}^{\prime}(\mathcal{D})$, i.e., $\sum_{\theta_p \in \Theta}\left(f_{\theta_p}^{\prime}(\mathcal{I})-\left(1-R_{\theta_p}\right)\right)=0$. Noting that $f_{\theta_p}^{\prime}(\mathcal{I}) \leq 1-R_{\theta_p}$ for all $\mathcal{I} \subseteq\mathcal{D}$, we have $f_{\theta_p}^{\prime}(\mathcal{I})=1-R_{\theta_p}$ for all $\theta_p \in \Theta$, i.e., $f_{\theta_p}(\mathcal{I}) \geq 1-R_{\theta_p}$ for all $\theta_p \in \Theta$. This completes the proof of the lemma.
\end{proof}

\subsection{Proof of Theorem \ref{thm:greedy}}
The performance guarantees presented in this result are that of the approximate submodular set covering problem in Theorem 9 of \cite{das2018approximate}, where 
\begin{align}
\gamma_{S^{NG},k^*(C)}(f) &= \gamma;\\
C &= z(\mathcal{D}) - z(\mathcal{\emptyset}) ;\\
    k &= c(\mathcal{I}_g);\\
    k^*(C) &= c(\mathcal{I}^*);\\
    f(S_{k-1}^{NG}) &= z(\mathcal{I}_g^{T-1}) - z(\mathcal{\emptyset}).
\end{align}
With probability at least $1-\delta$, Algorithm \ref{alg:greedy}, under $\Tilde{N}(\epsilon, \delta, L)$ observation samples, enjoys the same performance guarantees, where $\delta, \epsilon$ are specified by the central designer.

\subsection{Proof of Corollary \ref{coro:mcis}}
\begin{corollary*}
    Instate the hypothesis and notation of Theorem \ref{thm2}. As $t \to \infty$, we have the following: (a) $\mu_{\infty}^{\mathcal{I}}(\theta_q) = 0 \quad \forall \theta_q \notin F_{\theta_p}(\mathcal{I})$, and  (b) $\mu_{\infty}^{\mathcal{I}}(\theta_q) = \frac{1}{|F_{\theta_p}(\mathcal{I})|} \quad \forall \theta_q \in F_{\theta_p}(\mathcal{I})$. The near-optimal guarantees provided in Theorem \ref{thm:greedy} for Problem \ref{prob:dsrc} hold with probability 1 (a.a.s.).  
\end{corollary*}
\begin{proof}
    From Theorem \ref{thm2}, the following hold for $t \to \infty$:
        \begin{enumerate}[label=(\alph*)]
        \item  $\mu_t^{\mathcal{I}}(\theta_q) =  \mu_t^{\mathcal{I}}(\theta_p)  \hspace{5pt} \forall \theta_q \in F_{\theta_p}(\mathcal{I})$, and
        \item  $\mu_t^{\mathcal{I}}(\theta_q) \leq \exp{(-t(|K(\theta_p, \theta_q)-\epsilon|))} \hspace{5pt} \forall \theta_q \notin F_{\theta_p}(\mathcal{I})$;
    \end{enumerate}
    We have $\lim_{t \to \infty} \mu_t^{\mathcal{I}}(\theta_q) = 0 \quad \forall \theta_q \notin F_{\theta_p}(\mathcal{I}) $ and $\mu_t^{\mathcal{I}}(\theta_q) =  \mu_t^{\mathcal{I}}(\theta_p)  \hspace{5pt} \forall \theta_q \in F_{\theta_p}(\mathcal{I})$. Since $\sum_{\theta \in \Theta} \mu_t^{\mathcal{I}}(\theta) = 1$ for all $t$, we have $\lim_{t \to \infty}  \mu_{t}^{\mathcal{I}}(\theta_q)= \frac{1}{|F_{\theta_p}(\mathcal{I})|} \quad \forall \theta_q \in F_{\theta_p}(\mathcal{I})$.
    Since we have $\lim_{t \to \infty} \mu_t^{\mathcal{I}}(\theta_q) = 0 \quad \forall \theta_q \notin F_{\theta_p}(\mathcal{I}) $, the central designer will predict one of the hypotheses $\theta_q \in F_{\theta_p}(\mathcal{I})$ as the true state of the world, with probability 1. Therefore, the guarantees provided in Theorem \ref{thm:greedy} for Problem \ref{prob:dsrc} hold with probability 1. 
\end{proof}
\subsection{Proof of Theorem \ref{thm:greedy2}}
The performance guarantees presented in this result are that of the approximate submodular maximization under Knapsack constraints in Theorem 6 of \cite{das2018approximate}, where 
\begin{align}
\gamma_{S^{NG},k}(f) &= \gamma;\\
    f(S^{NG}) &= \Lambda(\mathcal{I}_g) - \Lambda(\mathcal{\emptyset});\\
    OPT &= \Lambda(\mathcal{I}^*) - \Lambda(\mathcal{\emptyset}).
\end{align}

With probability at least $1-\delta$, Algorithm \ref{alg:greedy2}, under $\Tilde{N}(\epsilon, \delta, L)$ observation samples, enjoys the same performance guarantees, where $\delta, \epsilon$ are specified by the central designer.

\subsection{Proof of Corollary \ref{coro:mpis}}
The construction of the proof and arguments are similar to those presented in Corollary \ref{coro:mcis}.

\subsection{Proof of Lemma \ref{lma:surr}}
We first begin by defining a submodular set function.
\begin{definition}[Submodular Set Function](\cite{nemhauser1978analysis})
\label{def:submodular}
 A set function $f:2^{\Omega }\to\mathbb{R}$ is submodular if it satisfies $f(X\cup\{j\})-f(X)\ge f(Y\cup\{j\})-f(Y)$, $\forall X\subseteq Y\subseteq \Omega $ and  $ \forall j\in \Omega \setminus Y$.
\end{definition}
We now establish the following result.
\begin{lemma*} The function $ g_{\theta_p}(\mathcal{I}):2^{\mathcal{D} }\to\mathbb{R}_{\geq0} $ is submodular for all $\theta_p \in \Theta$.
\end{lemma*}
\begin{proof}
    Recall that $g_{\theta_p}(\mathcal{I}) = 1 - \sum_{\theta_i \in F_{\theta_p}(\mathcal{I})} \xi_{pi}.$ Consider any $\mathcal{I}_1 \subseteq \mathcal{I}_2 \subseteq\mathcal{D}$ and any $j \in \mathcal{D} \backslash \mathcal{I}_2$. We have the following:
$$
\begin{aligned}
  g_{\theta_p}\left(\mathcal{I}_1 \cup\{j\}\right) - g_{\theta_p}\left(\mathcal{I}_1\right) 
= & \sum_{\theta_i \in F_{\theta_p}\left(\mathcal{I}_1\right)} \xi_{pi} - \sum_{\theta_i \in F_{\theta_p}\left(\mathcal{I}_1 \cup\{j\}\right)}  \xi_{pi}  \\
= & \sum_{\theta_i \in F_{\theta_p}\left(\mathcal{I}_1\right)} \xi_{pi} -\sum_{\theta_i \in F_{\theta_p}\left(\mathcal{I}_1\right) \cap F_{\theta_p}(j)} \xi_{pi} \\
= & \sum_{\theta_i \in  F_{\theta_p}\left(\mathcal{I}_1\right) \backslash \left(F_{\theta_p}\left(\mathcal{I}_1\right) \cap F_{\theta_p}(j)\right)} \xi_{pi}\\
=&\sum_{\theta_i \in F_{\theta_p}(\mathcal{I}_1) \backslash F_{\theta_p}\left(j\right)} \xi_{pi} .
\end{aligned}
$$
Note that the above arguments follow from  De Morgan's laws. Similarly, we  have
$$
g_{\theta_p}\left(\mathcal{I}_2 \cup\{j\}\right)- g_{\theta_p}\left(\mathcal{I}_2\right)=\sum_{\theta_i \in F_{\theta_p}(\mathcal{I}_2) \backslash F_{\theta_p}\left(j\right)} \xi_{pi} .
$$
Since $\mathcal{I}_1 \subseteq \mathcal{I}_2$, we have $F_{\theta_p}(\mathcal{I}_2) \backslash F_{\theta_p}\left(j\right) \subseteq F_{\theta_p}(\mathcal{I}_1) \backslash F_{\theta_p}\left(j\right)$. Thus,
$$
g_{\theta_p}\left(\mathcal{I}_1 \cup\{j\}\right) - g_{\theta_p}\left(\mathcal{I}_1\right) \geq g_{\theta_p}\left(\mathcal{I}_2 \cup\{j\}\right)-g_{\theta_p}\left(\mathcal{I}_2\right) .
$$
Since the above arguments hold for all $\theta_p \in \Theta$, the function $g_{\theta_p}(\cdot)$ is submodular for all $\theta_p \in \Theta$.
\end{proof}

\subsection{Proof of Corollary \ref{coro:submod_greedy}}
\begin{corollary*}
For Algorithm \ref{alg:greedy} (respectively Algorithm \ref{alg:greedy2}) applied to M-MCIS (respectively M-MPIS) Problem, the high-probability near-optimal guarantees provided in Theorem \ref{thm:greedy} (respectively Theorem \ref{thm:greedy2}) hold with $\gamma=1$. 
\end{corollary*}
\begin{proof}
    Consider the M-MCIS problem defined in (\ref{eq:mmcis}). Based on similar arguments as in Section \ref{sec:weak_submod}, we can transform the M-MCIS Problem into the set covering problem defined in (\ref{eq:submod_opt}), where the constraints $z(\cdot)$ are expressed in terms of $g_{\theta_p}(\cdot)$ and $R'_{\theta_p}$ instead of $f_{\theta_p}(\cdot)$ and $R_{\theta_p}$, as follows.

    \begin{equation}
        z(\mathcal{I}) = \sum_{\theta_p \in \Theta} \min \{ g_{\theta_p}(\mathcal{I}), 1-R'_{\theta_p}\}.
    \end{equation}
    
    It follows from Lemma \ref{lma:surr} that  $z(\cdot)$ is submodular. Since we have $\gamma \ge 1$ for submodular functions (\cite{das2018approximate}), the guarantees presented in Theorem \ref{thm:greedy} for Algorithm \ref{alg:greedy} hold with $\gamma=1$ for the M-MCIS problem.

    Now consider the M-MPIS problem. We now define the following optimization problem.
    \begin{equation}
\label{eq:mmpis_max}
\begin{aligned}
 \max _{\mathcal{I} \subseteq \mathcal{D}} & \hspace{5pt} \displaystyle \sum_{\theta_p \in \Theta} \left(1- \rho_{\theta_p}(\mathcal{I}) \right);
  \hspace{2pt} \textit { s.t. } \sum_{i \in \mathcal{I}} c_i \leq K.
\end{aligned}
\end{equation}
It is easy to verify that the problem defined in (\ref{eq:mmpis_max}) is equivalent to the problem defined in (\ref{eq:mmpis}), i.e., the information set $\mathcal{I} \subseteq \mathcal{D}$ that optimizes the problem in Equation (\ref{eq:mmpis_max}) is also the optimal solution to the Problem \ref{prob:mmpis}. It follows from Lemma \ref{lma:surr} that the objective function of the Problem (\ref{eq:mmpis_max}) is submodular. Since we have $\gamma \ge 1$ for submodular functions (\cite{das2018approximate}), the guarantees presented in Theorem \ref{thm:greedy2} for Algorithm \ref{alg:greedy2} hold with $\gamma=1$ for the M-MPIS problem.
\end{proof}
\vspace{-10pt}
\section{Limitations of Approximate Submodularity}
\label{app:limit}
In this section, we present insights into the effect of the penalty values on the submodularity ratio and performance bounds of greedy. Recall that the submodularity ratio is given by $\gamma = \underline{\xi}/ \Bar{\xi} $, where 
\begin{align}
\displaystyle \underline{\xi} = \min_{\theta_p \in \Theta} \left(\min_{\theta_i,\theta_j \in \Theta}  |\xi_{pi} - \xi_{pj}|\right);
    \Bar{\xi} = \max_{\theta_p \in \Theta} \left(\max_{\theta_i,\theta_j \in \Theta}  |\xi_{pi} - \xi_{pj}|\right).
\end{align}

Assumption 2 requires that the misclassification penalties are unique, i.e., $\xi_{pi} \neq \xi_{pj}$ for $i \neq j$, for all $\theta_p \in \Theta$. However, in many practical scenarios, one may have equal (or arbitrarily close) misclassification penalties, as in the Aerial Vehicle Classification task considered in this paper. This leads to a trivial lower bound on the submodularity ratio, i.e., $\gamma \geq 0$. In such cases, one may attempt to provide an improved lower bound on the submodularity ratio $\gamma$ using problem specific insights, which is strictly bounded away from 0. Now consider the following example. 

\noindent \textbf{Example 1: } Consider an instance of the MPIS Problem with $\Theta = \{\theta_1, \theta_2, \theta_3 \}$ and the penalty matrix given by $    \Xi = \begin{bmatrix}
0 & 0.5 & 0.5\\
0.5 & 0 & 0.5\\
0.5 & 0.5 &0
\end{bmatrix}.$
 Let the set of available information sources be $\mathcal{D} = \{ i_1, i_2 \}$, with $c_1=c_2=1.$ Let the true hypothesis be $\theta_p = \theta_1$. Without loss of generality, the arguments below hold true for the cases with $\theta_p=\theta_2$ and $\theta_p = \theta_3$. Let the observationally equivalent sets for the information sources be: $F_{\theta_1}(i_1) = \{\theta_1,\theta_2\}$ and $F_{\theta_1}(i_2) = \{\theta_1, \theta_3\}$.
 
 We will now characterize the submodularity ratio for this problem instance. We need to provide a lower bound on $\gamma$, which is given by 
\begin{equation}
    \gamma \leq \frac{\sum_{a \in A \setminus B} (f(\{a\} \cup B) - f(B))}{f(A \cup B) - f(B)},
\end{equation}
for all $A,B \subseteq \mathcal{D}$, where $0/0$ is defined as 1. In order to provide non-trivial bounds on the performance of greedy algorithm, $\gamma$ should be strictly bounded away from 0. However, consider the following: $A = \{i_1, i_2\}$ and $B = \emptyset$. It is easy to see that $\sum_{a \in A \setminus B} (f(\{a\} \cup B) - f(B)) = 0$ and $f(A \cup B) - f(B) = 0.5$, where $f(\mathcal{I}) = 1 - \max_{\theta_i \in F_{\theta_1}(\mathcal{I})} \xi_{1i}.$ Thus, it follows that the lower bound on $\gamma$ is zero, which is trivial. Therefore, there exist many instances of these problems where one may not be able to provide non-trivial guarantees for the near-optimality of the greedy algorithms. 

Instead, one can turn to an alternate optimization metric, which can provide non-trivial and near-optimal bounds for approximating the optimal solution (see Section \ref{sec:alt_form}). However, establishing  performance guarantees for optimizing the alternate metric (i.e., total penalty of misclassification) in place of the actual metric (i.e., maximum penalty of misclassification) is still an open question.

\section{Additional Experiments}
\label{app:exp}
\subsection{Finite Sample Convergence of Bayesian beliefs}
We consider an instance of Problem \ref{prob:dsrc} with number of states $\Theta = 10$ and show the finite sample convergence of beliefs. We randomly sample the observationally equivalent set $F_{\theta_p}(\mathcal{I})$ for an information set $\mathcal{I} \in \mathcal{D}$. We start with a uniform prior and apply the Bayesian update rule as in (\ref{eq:bayes_full}). Figure \ref{fig:finite_sample} shows convergence of beliefs under $50$ observation samples. It can be verified that the beliefs over the states not in the observationally equivalent set of the true state $\theta_p = 0$ get arbitrarily close to zero and the beliefs over the states which are observationally equivalent to the true state $\theta_p = 0$,  i.e., $\{1,5,7,8\}$ are all equal to $0.2$, validating the results presented in Theorem \ref{thm2}.
\begin{figure}[htbp!]
    \centering
    \includegraphics[width=120pt]{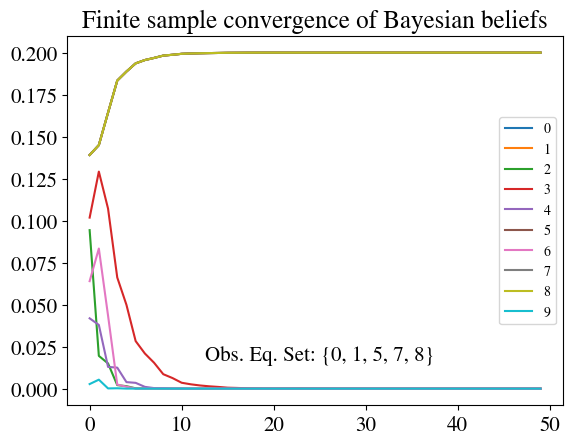}
    \caption{Finite Sample Convergence of Bayesian Beliefs}
    \label{fig:finite_sample}
\end{figure}
\vspace{-0.5cm}
\subsection{Performance of Greedy Algorithm for Varying Submodularity Ratios}
In this section, we present numerical simulations to evaluate the performance of Algorithm \ref{alg:greedy} (for MCIS Problem) and Algorithm \ref{alg:greedy2} (for MPIS problem) with varying submodularity ratios of the problem instances. Note that the submodularity ratio only depends on the penalty matrix entries.  We generate a set of 5 penalty matrices $\mathcal{P} = \{ \Xi_1, \Xi_2, \Xi_3, \Xi_4, \Xi_5 \}$, corresponding to submodularity ratio bounds $\gamma \in \{0.1, 0.2, 0.3, 0.4, 0.5 \}$, respectively. In other words, these penalty matrices ensure that the submodularity ratio for these instances are at least $\gamma$. We set  $|\mathcal{D}| = 10$. We consider the scenario with infinite observation samples, where the constraints in (\ref{eq:submod_opt}) can be completely specified by $F_{\theta_p}(i)$ for all $\theta_p \in \Theta$ and for all $i \in \mathcal{D}$, which capture the underlying likelihood functions $\ell_i(\cdot|\theta_p)$. We randomly generate the set $F_{\theta_p}(i)$ for all $i \in \mathcal{D}$ and for all $\theta_p \in \Theta$. For each penalty matrix $\Xi_i \in \mathcal{P}$, we generate 20 random instances of the MCIS problem with varying penalty thresholds and information set costs. We run Algorithm \ref{alg:greedy} to generate the greedy information set and find the optimal set through brute-force search. We plot the distributions of the ratio of cost of information set generated by greedy to that of the optimal in Figure \ref{fig:submod_exp}(a). Similarly, for each penalty matrix $\Xi_i \in \mathcal{P}$ we generate 20 random instances of the MPIS problem with varying budgets and information set costs. We run Algorithm \ref{alg:greedy2} to generate the greedy information set and find the optimal set through brute-force search. We plot the distributions of the ratio of greedy utility to that of the optimal in Figure \ref{fig:submod_exp}(b). We observe that as the submodularity ratio increases (i.e., the objective function is closer to being submodular) the performance of greedy improves.
\begin{figure}[!ht]
    \centering
    \subfigure[]{\includegraphics[width=0.3\textwidth]{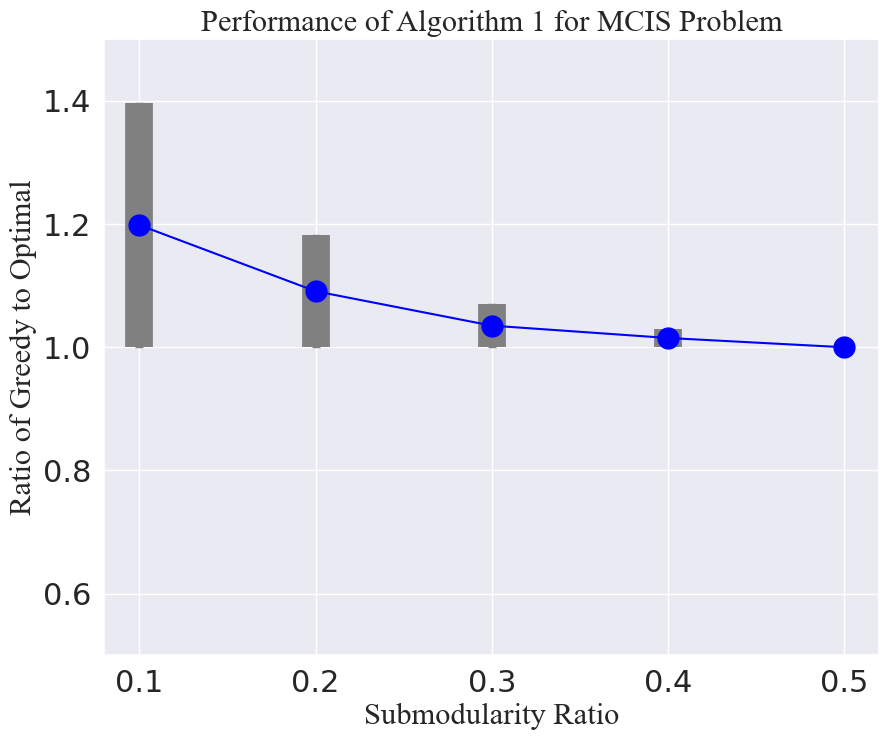}} 
    \subfigure[]{\includegraphics[width=0.3\textwidth]{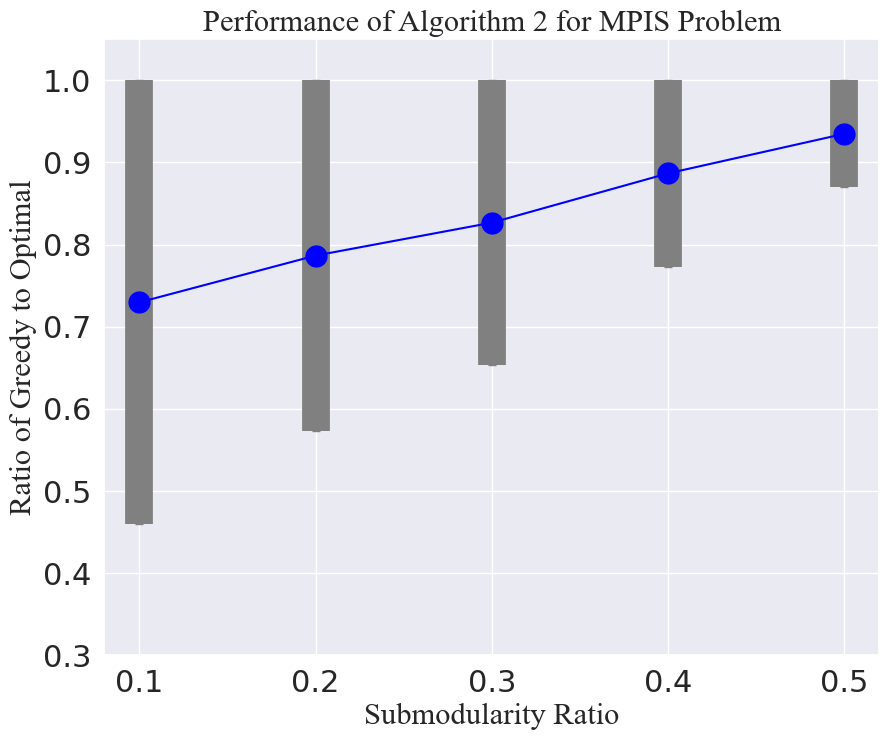}} 
    \caption{Performance of greedy algorithm for information set selection for varying submodularity ratios (a) Algorithm \ref{alg:greedy} (Problem \ref{prob:dsrc}), (b) Algorithm \ref{alg:greedy2} (Problem \ref{prob:mpis}).}
    \label{fig:submod_exp}
\end{figure}
\vspace{-1cm}
\subsection{Modified Information Set Selection Problems  }
We evaluate the performance of the greedy algorithms over random instances of the M-MCIS and M-MPIS problems. For M-MCIS problem, we generate 50 random instances, where for each instance, we set the total number of data sources $\mathcal{D}$ to be 10 and for each source $i \in \mathcal{D}$, the selection cost $c_i$ is drawn uniformly from $\{1,2,\hdots,10\}$. We generate a random row-stochastic penalty matrix $\Xi \in \mathbb{R}^{|\Theta| \times |\Theta|}$ with $|\Theta| = 20$. We consider a uniform prior $\mu_0(\theta) = 1/|\Theta|$ and set a penalty threshold of $R'/|\Theta|$, where $R$ is drawn randomly from  $\{1, \hdots, |\Theta|-1\}$ for all $\theta_p \in \Theta$. We consider the scenario with infinite observation samples, where the constraints can be completely specified by $F_{\theta_p}(i)$ for all $\theta_p \in \Theta$ and for all $i \in \mathcal{D}$, which capture the underlying likelihood functions $\ell_i(\cdot|\theta_p)$. We randomly generate the set $F_{\theta_p}(i)$ for all $i \in \mathcal{D}$ and for all $\theta_p \in \Theta$. In Figure \ref{fig:figure}(a), we plot the ratio of cost of the greedy information set to that of the optimal information set. Similarly, we generate 50 random instances of the M-MPIS problem. Here, we set uniform costs for the data sources.  For each instance, we set $|\mathcal{D}|=10$ and the randomly sample the selection budget from $\{1,2, \hdots,10\}$. In Figure \ref{fig:figure}(b), we plot the ratio of total penalty of the greedy information set to that of the optimal information set. We observe from Figure \ref{fig:figure} that the greedy algorithms exhibit near-optimal performance, which aligns with the theoretical guarantees provided by Corollary \ref{coro:submod_greedy}.
\begin{figure}[!ht]
    \centering
    \subfigure[]{\includegraphics[width=0.3\textwidth]{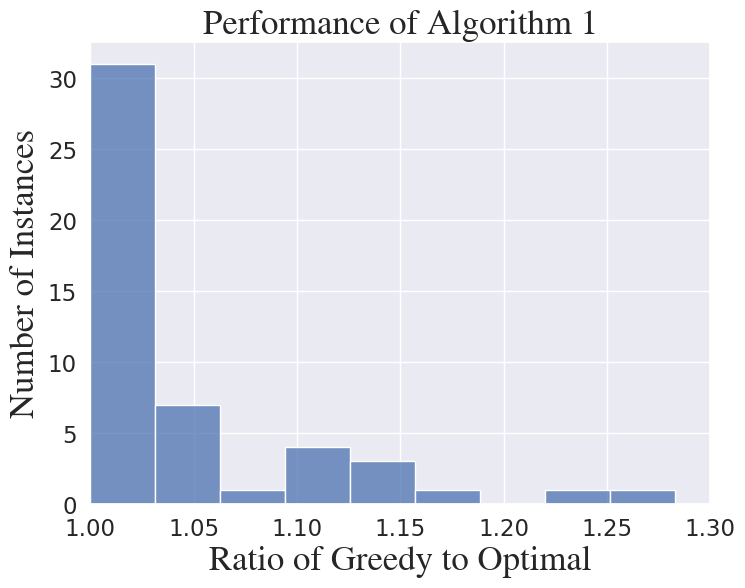}} 
    \subfigure[]{\includegraphics[width=0.3\textwidth]{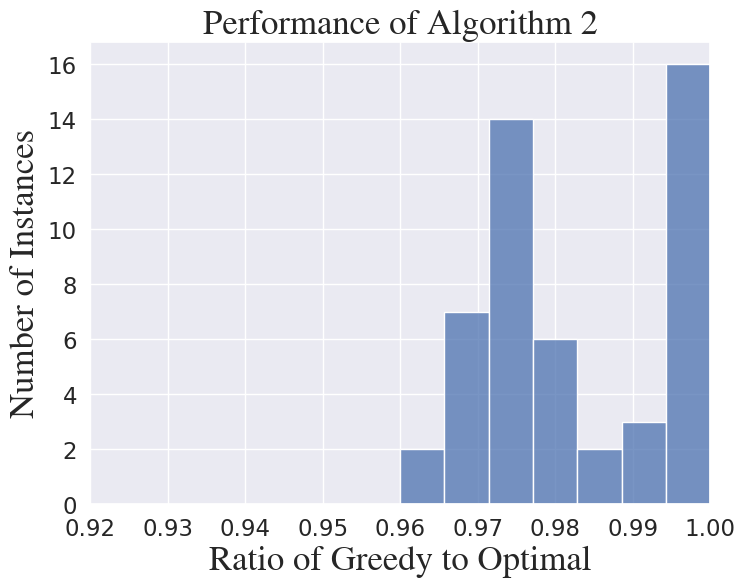}} 
    \caption{Performance of Greedy (a) Algorithm \ref{alg:greedy} (Problem \ref{prob:mmcis}), (b) Algorithm \ref{alg:greedy2} (Problem \ref{prob:mmpis}).}
    \label{fig:figure}
\end{figure}
\end{document}